\documentclass[accepted]{style}
\usepackage[american]{babel}

\usepackage{natbib}
\usepackage{mathtools}
\usepackage{nicefrac}
\usepackage{amssymb}
\usepackage{booktabs}
\usepackage{tabularx}
\usepackage{ulem}
\usepackage{tikz} 
\usepackage{amsthm}
\usepackage{cancel}

\usepackage{soul}
\setulcolor{black}

\newtheorem{proposition}{Proposition}
\newtheorem{theorem}{Theorem}

\usepackage{array}
\newcolumntype{L}{>{\raggedright\arraybackslash}X}
\newcolumntype{C}{>{\centering\arraybackslash}X}
\newcolumntype{R}{>{\raggedleft\arraybackslash}X}

\usepackage{multirow}
\usepackage{float}

\usepackage{pgfplots}
\pgfplotsset{compat=1.18}
\usetikzlibrary{patterns,patterns.meta,calc}

\title{Between Resolution Collapse and Variance Inflation:\\ Weighted Conformal Anomaly Detection in Low-Data Regimes}

\author[1]{\href{mailto:<oliver.hennhoefer@h-ka.de>}{Oliver~Hennhöfer}{}}
\author[1]{Christine~Preisach}
\affil[1]{%
    Intelligent Systems Research Group (ISRG)\\
    Karlsruhe University of Applied Sciences\\
    Germany
}
  
\begin{document}
\maketitle

\begin{abstract}
 Standard conformal anomaly detection provides marginal finite-sample guarantees under the assumption of exchangeability . However, real-world data often exhibit distribution shifts, necessitating a weighted conformal approach to adapt to local non-stationarity. We show that this adaptation induces a critical trade-off between the minimum attainable $p$-value and its stability. As importance weights localize to relevant calibration instances, the effective sample size decreases. This can render standard conformal $p$-values overly conservative for effective error control, while the smoothing technique used to mitigate this issue introduces conditional variance, potentially masking anomalies. We propose a continuous inference relaxation that resolves this dilemma by decoupling local adaptation from tail resolution via continuous weighted kernel density estimation. While relaxing finite-sample exactness to asymptotic validity, our method eliminates Monte Carlo variability and recovers the statistical power lost to discretization. Empirical evaluations confirm that our approach not only restores detection capabilities where discrete baselines yield zero discoveries, but outperforms standard methods in statistical power while maintaining valid marginal error control in practice.
\end{abstract}

\section{Introduction}\label{sec:introduction}

Anomaly detection aims to identify observations that deviate significantly from the majority of observations or do otherwise not \textit{conform} to an expected state of normality, indicating a distinct underlying data-generating mechanism at work \cite{Hawkins1980}. Yet, standard  detection approaches often lack statistical guarantees regarding the false alarm rate, which is problematic in safety-critical applications. Conformal Anomaly Detection (CAD) addresses this by offering a distribution-free framework to transform heuristic anomaly scores into valid $p$-values, enabling False Discovery Rate (FDR) control procedures \cite{Bates2023}.

Conformal validity assumes data exchangeability. In dynamic environments where the data distribution shifts over time, this assumption is violated. Weighted conformal approaches mitigate distribution shift by assigning higher importance to calibration samples resembling current test instances via likelihood ratios (\textit{covariate shift adaptation}). This localization induces a critical dilemma. As weights concentrate on a smaller effective sample size, empirical $p$-values become discretely coarse, leading to \textit{resolution collapse}, where the minimum $p$-value fails to meet the discovery threshold required for e.g. the Benjamini–Hochberg (BH) procedure \cite{Hochberg1995}.

While standard conformal theory proposes \textit{randomized smoothing} to resolve this granularity \cite{Jin2025}, we demonstrate that this theoretical fix comes at a practical cost in weighted regimes. When a test point is assigned a large weight (common under shift), smoothing introduces significant uniform noise to maintain exact validity. We show that this \textit{variance inflation} decreases the signal-to-noise ratio and degrades statistical power. Consequently, practitioners are (at worst) trapped between the zero power of the discrete estimator (due to inflated minimum attainable $p$-values) and low statistical efficiency of the randomized estimator (due to noise masking).

In this work, we address this conflict between local adaptation, resolution, and stability. Our contributions are:
\begin{itemize}
   \item \textbf{The Resolution--Variance Dilemma:} We formalize two failure modes of weighted CAD. As weight localization strength under covariate shift adaptation increases, the standard weighted conformal $p$-value method exhibits lower-bound $p$-values (\textit{resolution collapse}), while the randomized variant shows rejection inconsistency (\textit{variance inflation}). Both can severely reduce statistical power, particularly in low-data regimes.
    \item \textbf{Stabilized Continuous Inference:} We propose a continuous inference relaxation that decouples local adaptation from tail resolution via continuous weighted kernel density estimation. This approach eliminates the lower $p$-value bound without introducing the same degree of Monte Carlo noise of the randomized approach.
    \item \textbf{Empirical Validation:} We demonstrate that our approach restores detection capabilities in pathological regimes where discrete baselines yield zero discoveries, and significantly outperforms randomized baselines in statistical power by mitigating variance, all while maintaining valid marginal error control.
\end{itemize}

\section{Prior Works}

Conformal prediction under covariate shift was established by \citet{Tibshirani2019}, introducing the weighted exchangeability framework to correct for distribution shifts using importance weights (\textit{Radon–Nikod{\'y}m}). For CAD under covariate shift, \citet{Jin2025} extended this to multiple testing via \textit{Weighted Conformalized Selection} (WCS), 
necessitated by the failure of weighted conformal $p$-values to satisfy Positive Regression Dependence on a Subset (PRDS), under which BH guarantees FDR control.

\section{Preliminaries}

We consider the unsupervised anomaly detection setting. Let $\mathcal{D}_{\text{cal}} = \{z_1, \dots, z_N\}$ be a calibration set of $N$ observations drawn from a training distribution $P$. We evaluate a test batch $Z_{\text{test}} = \{z_{N+1}, \dots, z_{N+m}\}$ of $m$ instances drawn from a (possibly shifted) distribution $Q$. We fit a scoring function $s: \mathcal{Z} \to \mathbb{R}$ on $P$, where larger scores indicate greater deviation from normality. Let $s_i = s(z_i)$ denote the score for the $i$-th calibration observation, $i \in \{1,\ldots,N\}$.

\subsection{Weighted Conformal Anomaly Detection}

Standard conformal prediction assumes exchangeability between calibration and test data ($P = Q$). To accommodate covariate shift, we employ the weighted conformal framework \cite{Tibshirani2019}, which reweighs observations by the likelihood ratio $w(z) = dQ/dP(z)$ as estimated by a density estimator, see Section~\ref{subsec:weight_estimation}.

There are two standard approaches to constructing the weighted $p$-value for a test point $z_j$ with score $s_j$:

\paragraph{1. The Deterministic Estimator.}
The standard discrete $p$-value is conservative and includes the test point weight $w_j$ in the numerator to ensure validity without randomization:
\begin{equation}\label{eq:weighted_ecdf_cons}
    p_j^{\text{discrete}} = \frac{\sum_{i=1}^N w_i \,\mathbb{I}(s_i \ge s_j) + w_j}{\sum_{i=1}^N w_i + w_j}.
\end{equation}
This guarantees marginal super-uniformity \mbox{$\mathbb P\!\left(p_j^{\mathrm{disc}}\le \alpha\right)\le \alpha,\qquad \forall \alpha\in[0,1]$}, with a minimum attainable $p$-value $p_j^{\mathrm{disc}}=\nicefrac{w_j}{\sum_{i=1}^N w_i+w_j}$, even if $s_j \to \infty$.

\paragraph{2. The Randomized Estimator.}
To remove discretization effects, weighted conformal $p$-values introduce auxiliary randomness $U_j\sim\mathrm{Unif}[0,1]$ as in \citet{Jin2025}. In the unweighted exchangeable setting ($w\equiv 1$), the standard randomized conformal $p$-value is (marginally) valid and, under continuity/no-ties conditions, is exactly $\mathrm{Unif}[0,1]$ under the null. In the weighted covariate-shift setting, the same randomized construction yields marginal super-uniformity:

\begin{equation}\label{eq:weighted_ecdf_rand}
\begin{aligned}
p_j^{\text{rand}}
&=
\frac{\sum_{i=1}^N w_i \,\mathbb{I}(s_i > s_j)}
     {\sum_{i=1}^N w_i + w_j}
\\
&\quad
+
\frac{U_j\!\left(w_j + \sum_{i=1}^N w_i \,\mathbb{I}(s_i = s_j)\right)}
     {\sum_{i=1}^N w_i + w_j}
\text{.}
\end{aligned}
\end{equation}

The randomization spreads the mass of the test point’s own weight $w_j$, together with any calibration mass tied at $s_j$, uniformly over an interval. Hence, even if there are no calibration ties (e.g., for extreme scores outside the calibration range), the term $U_j w_j$ randomizes $p_j^{\text{rand}}$. If the scores are continuous (so that $\sum_i w_i \mathbb{I}(s_i = s_j)=0$), \eqref{eq:weighted_ecdf_rand} reduces to the simpler expression with $U_j w_j$. In all cases, the smoothed estimator can take values arbitrarily close to zero, eliminating the lower $p$-value bound of the deterministic estimator.

\subsection{False Discovery Rate Control}

For the batch \(Z_{\text{test}}\), we test \(m\) null hypotheses \mbox{\(H_{0,j}:\ z_j \text{ is an inlier}\)} and aim to control the FDR at level \(\alpha\).

\paragraph{Multiple Testing.}
In the unweighted setting ($w \equiv 1$), we apply the BH procedure. In the weighted setting, we apply WCS, which wraps the weighted $p$-values (discrete or randomized) to guarantee finite-sample FDR control despite complex dependencies induced by weight estimation. WCS relies on a self-consistency condition where the number of rejected hypotheses must support the rejection threshold.

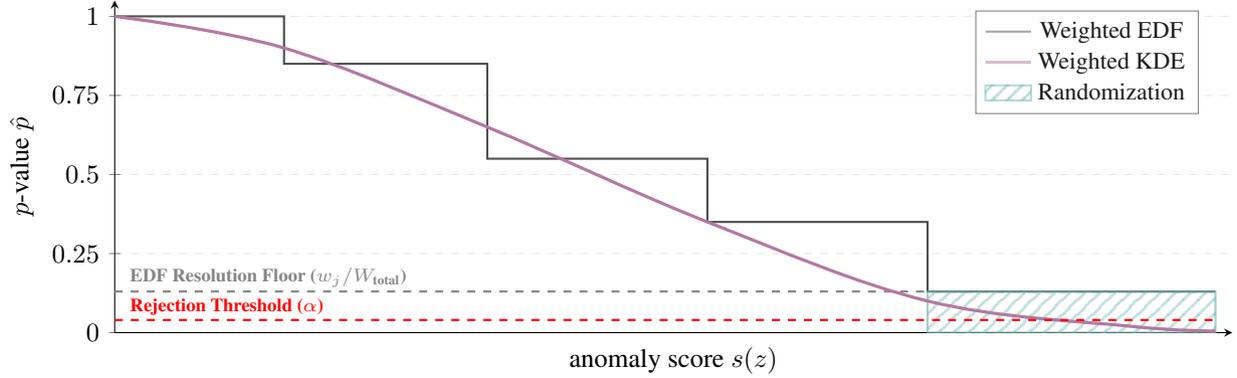
\begin{figure*}[t]
\centering
\begin{tikzpicture}
\begin{axis}[
    width=0.95\textwidth,
    height=6cm,
    xlabel={anomaly score $s(z)$},
    ylabel={$p$-value $\hat{p}$},
    xmin=0, xmax=6.6,
    ymin=0, ymax=1.05,
    axis lines=left,
    ytick={0, 0.25, 0.5, 0.75, 1},
    xtick=\empty, % Hide x ticks for schematic look
    legend pos=north east,
    legend style={font=\footnotesize, fill opacity=0.9, draw opacity=0.5},
    grid=major,
    grid style={dashed, gray!20},
]

% Define Colors based on your document style
\definecolor{colCons}{RGB}{80, 80, 80}    % Gray for Conservative
\definecolor{colRand}{RGB}{114, 183, 178} % Teal for Randomized
\definecolor{colKDE}{RGB}{178, 121, 162}  % Purple for KDE

% --- 1. The Resolution Floor (Dashey line) ---
% The limit of the conservative estimator is w_test / W_total
\draw[dashed, thick, colCons!70] (axis cs:0,0.13) -- (axis cs:6.5,0.13) 
    node[anchor=south west, pos=0.005, yshift=-2pt, font=\scriptsize] {\textbf{EDF Resolution Floor ($w_j / W_{\text{total}}$)}};

% --- 2. The Conservative Weighted EDF (Step Function) ---
% Steps are irregular to represent unequal weights
\addplot[
    const plot,
    thick,
    colCons,
] coordinates {
    (0,1.0) 
    (1.0, 1.0) (1.0, 0.85) 
    (2.2, 0.85) (2.2, 0.55) % Big drop (large weight calibration point)
    (3.5, 0.55) (3.5, 0.35) 
    (4.8, 0.35) (4.8, 0.13) % Hits the floor
    (6.5, 0.13)             % Stays at floor
};
\addlegendentry{Weighted EDF}

% --- 3. The Randomized Range (Variance Inflation) ---
\draw[
    thick,                  % Thickness of the border
    draw=colRand,           % Color of the border
    opacity=0.65,
    pattern={Lines[angle=45, distance={4pt}, line width={1pt}]}, % Thicker pattern lines
    pattern color=colRand   % Color of the pattern lines (no !50)
] (axis cs:4.8,0) rectangle (axis cs:6.5,0.13);

% --- 4. The Continuous KDE (Ours) ---
% Smooth curve that cuts through steps and decays to 0
\addplot[
    smooth,
    very thick,
    colKDE,
    tension=0.6
] coordinates {
    (0, 1.0)
    (1, 0.90)
    (2.2, 0.65)
    (3.5, 0.35)
    (4.8, 0.10)
    (6.0, 0.02)
    (6.5, 0.005)
};

\addlegendentry{Weighted KDE}

% Dummy plot for legend
\addplot[
    area legend, 
    draw=colRand, 
    thick,
    pattern={Lines[angle=45, distance={4pt}, line width={1pt}]},
    pattern color=colRand
] coordinates {(0,0)};
\addlegendentry{Randomization}

% --- Annotations ---

% Annotation: Rejection Threshold
\draw[dashed, thick, red] (axis cs:0,0.04) -- (axis cs:6.5,0.04) 
    node[anchor=south west, pos=0.005, yshift=-2pt, font=\scriptsize] {\textbf{Rejection Threshold ($\alpha$)}};

% Annotation: The "Trap"
\node[coordinate] (trap_point) at (axis cs: 5.8, 0.13) {};
\node[coordinate] (kde_point) at (axis cs: 5.8, 0.03) {};

\end{axis}
\end{tikzpicture}
\caption{\textbf{The Resolution--Variance Dilemma.} Under distribution shift, high importance weights create large steps in the conservative estimator, imposing a \textit{resolution floor} of minimum attainable $p$-values that prevents rejection even for extreme scores. Standard randomization resolves the floor but introduces \textit{variance inflation} via noise that can mask the signal. The proposed Weighted KDE decouples resolution from sample size, enabling rejection below the floor without stochastic noise.}
\label{fig:dilemma_schematic}
\end{figure*}

\section{The Dilemma}
\label{sec:dilemma}
In dynamic environments, the weights $w(z)$ adapt to local distribution shifts. We demonstrate that this adaptation forces a critical trade-off between resolution (\textit{rejection ability}) and stability (\textit{rejection consistency}), see Figure~\ref{fig:dilemma_schematic}.

\paragraph{Failure Mode A: Resolution Collapse (Discretization).}
For the conservative estimator (Eq.~\ref{eq:weighted_ecdf_cons}), the smallest $p$-value is lower-bounded by the test point’s relative weight,
\[
p^{\text{discrete}}_{j} \;\ge\; \frac{w_j}{\sum_{i=1}^N w_i + w_j}.
\]

\iffalse
To quantify when discretization prevents discoveries, we compare the minimum attainable $p$-values to the corresponding \textit{oracle} BH rejection threshold. If the BH procedure were to reject exactly the \(k\) truly anomalous hypotheses among \(m\) tests at significance level \(\alpha\), this threshold would be \((k/m)\alpha\). We therefore define the \textit{Detectability Ratio} \(\delta\) as the ratio between the minimum attainable $p$-value (achieved when $s_j > \max_{i} s_i$) and the BH rejection threshold:

\begin{equation}
    \delta_j
    \;=\;
    \frac{\left(\nicefrac{w_j}{\sum_{i=1}^N w_i + w_j}\right)}
         {(k/m)\cdot \alpha }.
\end{equation}

When weights localize so that $\delta_j>1$, even a maximally extreme score cannot yield a $p$-value below the cutoff, rendering the discrete weighted $p$-value is \emph{incapable} of supporting $k$ rejections under BH (i.e., strict power loss). In the unweighted setting ($w\equiv 1$), the floor reduces to $1/(N+1)$.
\fi

To quantify when discretization prevents discoveries, we compare the minimum attainable $p$-value (for $s_j>\max_i s_i$) to a \textit{heuristic} BH scale at target FDR level $\alpha$. If BH ends up making $R$ rejections, its cutoff is $(R/m)\alpha$. Motivated by this, we define a detectability ratio $\delta(r)$ relative to a \textit{putative} rejection count $r\in\{1,\dots,m\}$ as
\begin{equation}\label{eq:detectability_ratio}
\delta_j(r)
=
\frac{\displaystyle \nicefrac{w_j}{\sum_{i=1}^N w_i + w_j}}
{\displaystyle (r/m)\cdot\alpha}.
\end{equation}
If $\delta_j(r)>1$, then even the most extreme right-tail score ($s_j>\max_i s_i$) cannot produce $p_j^{\mathrm{disc}}\le (r/m)\alpha$, so BH cannot reach $r$ rejections using the discrete weighted \mbox{$p$-values} on this realized weight configuration.

In the unweighted setting ($w\equiv 1$), the minimum attainable $p$-value in the numerator of Eq.~\ref{eq:detectability_ratio} reduces to $1/(N+1)$.

\iffalse
\paragraph{Failure Mode B: Variance Inflation (Randomization).}
The randomized estimator (Eq.~\ref{eq:weighted_ecdf_rand}) resolves the floor issue ($\min p^{\text{rand}} \to 0$) by introducing a stochastic noise term:

\begin{equation}
    \text{Noise}_j
    \;=\;
    U_j \cdot \left( \frac{w_j}{\sum_{i=1}^N w_i + w_j} \right),
    \qquad
    U_j \sim \mathrm{Unif}(0,1).
\end{equation}

Conditional on the scores, $p_j^{\text{rand}}$ is uniform on the interval

\[
\left[
\begin{aligned}
&\frac{\sum_i w_i \mathbb{I}(s_i > s_j)}
       {\sum_i w_i + w_j},
\\[0.5ex]
&\frac{\sum_i w_i \mathbb{I}(s_i > s_j)
      + w_j
      + \sum_i w_i \mathbb{I}(s_i = s_j)}
      {\sum_i w_i + w_j}
\end{aligned}
\right]
\]

so the noise amplitude is the interval width. In the extreme-tail, no-ties case, this reduces to $\text{Unif}[0, w_j/(\sum_i w_i + w_j)]$, which yields the variance expression above. With ties, the width increases by the tied calibration weight.

Under distribution shift, $w_j$ can become large relative to the effective calibration sum. In the extreme‑tail, no‑ties case, the $p$-value for a clear anomaly effectively becomes a draw from $\text{Unif}[0, \text{Noise}_j]$. If $\text{Noise}_j$ is large, a true anomaly will fail to be rejected in a significant fraction of random realizations, simply due to an unlucky draw of $U_j$. We term this circumstance \textit{Variance Inflation}.
\fi

\paragraph{Failure Mode B: Variance Inflation (Randomization).}
The randomized estimator (Eq.~\ref{eq:weighted_ecdf_rand}) replaces the discrete step at the test point by spreading a deterministic mass uniformly over an interval. Let $W_{\mathrm{cal}}=\sum_{i=1}^N w_i$ so that
\[
W_{\mathrm{total}}=W_{\mathrm{cal}}+w_j,\qquad
W_{=}(s_j)=\sum_{i=1}^N w_i\,\mathbb I(s_i=s_j).
\]
Then conditional on the realized scores and respective weights, $p_j^{\mathrm{rand}}$ is uniform on an interval of deterministic width
\[
\mathrm{width}_j=\frac{w_j+W_{=}(s_j)}{W_{\mathrm{total}}}.
\]
Equivalently, one may write the randomization contribution as
\[
\mathrm{Rand}_j = U_j\cdot \mathrm{width}_j,\qquad U_j\sim\mathrm{Unif}(0,1),
\]
so the \textit{amplitude} is $\mathrm{width}_j$ and the \textit{random noise} is the uniform draw on $[0,\mathrm{width}_j]$.

In the extreme-right-tail, no-ties case $s_j>\max_i s_i$ and with $W_{=}(s_j)=0$, we have
\[
p_j^{\mathrm{rand}} = U_j\cdot \frac{w_j}{W_{\mathrm{total}}}
\sim \mathrm{Unif}\!\left[0,\frac{w_j}{W_{\mathrm{total}}}\right].
\]
Thus, if the relative test weight $w_j/(W_{\mathrm{cal}}+w_j)$ is large, the randomized $p$-value exhibits substantial conditional variability even for extremely large scores. A true anomaly might fail to be rejected, simply due to an unlucky draw of $U_j$.

\iffalse
\subsection{Effective Sample Size}
Both failure modes are driven by the reduction in effective sample size \cite{Kish1965}, defined as
\begin{equation}
    N_{\text{eff}} = \frac{\left( \sum_{i=1}^{N} w_i \right)^2}{\sum_{i=1}^{N} w_i^2}.
\end{equation}
As $N_{\text{eff}}$ drops, the relative mass of the test point increases. This simultaneously raises the floor for the conservative estimator (causing collapse) and amplifies the noise amplitude for the randomized estimator (causing variance). Resolving this dilemma requires an approach that decouples $N_{\text{eff}}$ from $p$-value granularity without stochastic noise, motivating the continuous inference procedure introduced in Section~\ref{sec:method}.
\fi

\subsection{Effective sample size}
Key driver of both \textit{collapse} (discrete floor) and \textit{variance inflation} (randomization amplitude) is the concentration of calibration weights. Following \citet{Kish1965}, for nonnegative weights $w_1,\dots,w_N$, define the effective sample size as
\begin{equation}
    N_{\text{eff}} = \frac{\left( \sum_{i=1}^{N} w_i \right)^2}{\sum_{i=1}^{N} w_i^2}.
\end{equation}
When $N_{\text{eff}}$ is small, a few calibration instances carry most of the mass. 
The test point's relative weight $w_j/W_{\mathrm{total}}$ tends to increase, raising the minimum attainable conformal $p$-value. The randomized conformal $p$-value interval width scales like $\mathrm{width}_j$, increasing their conditional variance.

\section{Continuous Weighted Conformal Inference}
\label{sec:method}

We propose an inference procedure that first adapts to covariate shift via density ratio estimation, and then constructs high-resolution $p$-values using weighted kernel density estimation (KDE). By modelling the underlying score distribution rather than counting discrete exceedances, we decouple the ability to reject from the effective sample size.

\subsection{Covariate Shift Adaptation}
\label{subsec:weight_estimation}

To account for distribution shift between the calibration distribution $P$ and the test distribution $Q$, we estimate the likelihood ratio $w(z) = dQ/dP(z)$. Following the \textit{density ratio trick} (i.e., importance weighting)  \cite{Sugiyama2008,Sugiyama2012}, we reduce this to a probabilistic classification problem. We train a classifier (e.g. Random Forest) to discriminate between calibration samples ($Y=0$, $X \sim P$) and test samples ($Y=1$, $X \sim Q$). Weights for any test or calibration instance $z_i$ are given by
\begin{equation}
    \hat{w}_i = \frac{\hat{P}(Y=1 \mid z_i)}{\hat{P}(Y=0 \mid z_i)} \cdot \frac{N_{\text{cal}}}{N_{\text{test}}}.
\end{equation}
We apply winsorization to mitigate the effect of extreme importance weights from limited support overlap between $P$ and $Q$, clipping weights to the $[\gamma, 1-\gamma]$ quantiles of the observed weight distribution ($\gamma=0.05$). The procedure assumes covariate shift is invariant across domains and the support of $Q$ is sufficiently contained in that of $P$.

\subsection{Weighted Kernel Density Estimation}
\label{subsec:kde_pvalue}

Standard weighted conformal $p$-values are discrete. Canonical weighted conformal methods add randomization to interpolate weighted ranks yielding continuous $p$-values, introducing Monte Carlo noise. To avoid discretization and Monte Carlo randomization, we instead approximate the weighted score distribution using kernel density estimation.

Let $s_i=s(z_i)$ be calibration scores and $\hat w_i\ge 0$ be calibration weights. Define the weighted KDE
\begin{equation}
\hat f^{\,w}(s)
=
\frac{1}{h\sum_{k=1}^N \hat w_k}\sum_{i=1}^N \hat w_i\,
K\!\left(\frac{s-s_i}{h}\right),
\end{equation}

where $K(\cdot)$ is a symmetric kernel function satisfying $\int_{\mathbb R}K(u)\,du=1$, with bandwidth $h>0$ (e.g., the Gaussian kernel). Let
\[
\Phi_K(t)=\int_{-\infty}^{t}K(u)\,du
\]
denote the CDF associated with $K$.

\paragraph{Bandwidth Selection.}
We select bandwidth $h$ via leave-one-out cross-validation to maximize the weighted log-likelihood. This data-driven approach adapts smoothness to $N_{\text{eff}}$, balancing the risk of over-smoothing (bias) against spurious modes (variance).

\subsection{Continuous $p$-value Construction}

For a test score $s_{j}$, define the right-tail $p$-value under the fitted weighted calibration score density $\hat f^{\,w}$ as
\[
\hat p_{\mathrm{KDE}}(s_{j})=\int_{s_{j}}^{\infty} \hat f^{\,w}(t)\,dt.
\]
With $\Phi_K(t)=\int_{-\infty}^t K(u)\,du$, this equals
\begin{equation}
\label{eq:kde_pvalue}
\hat p_{\mathrm{KDE}}(s_{j})
=
1-\sum_{i=1}^N \frac{\hat w_i}{\sum_{k=1}^N \hat w_k}\,
\Phi_K\!\left(\frac{s_{j}-s_i}{h}\right).
\end{equation}

\paragraph{Granularity and Stability.}
The estimator $\hat{p}_{\text{KDE}}(s)$ maps continuously to $[0, 1]$, eliminating the discrete floor $w_{j} / W_{\text{total}}$ inherent in weighted rank-based $p$-values. Unlike the canonical randomized weighted conformal $p$-value---which uses \mbox{$U_j\sim \mathrm{Unif}[0,1]$} to randomize the weighted rank (not only to break ties, but also to interpolate in extreme tail/out-of-range cases)---our KDE smoothing removes the auxiliary tie-breaking and interpolation randomness ($U_j$) used by randomized conformal $p$-values, yielding deterministic $p$-values conditional on the fitted KDE. This does not remove statistical estimation errors (finite-sample KDE error, weight-estimation error, and bandwidth-selection variability), but it removes the additional noise by $U_j$.

\subsection{Integration with Multiple Testing}
\label{subsec:integration}

The proposed continuous estimator $\hat{p}$ acts as a continuous surrogate for empirical conformal $p$-values in downstream multiple-testing pipelines. However, we emphasize that it is not guaranteed to inherit the finite-sample conformal validity properties of weighted conformal $p$-values.

\begin{itemize}
    \item \textbf{Unweighted Regime:} In the absence of shift, substituting continuous estimates into the BH procedure mitigates discretization conservatism in small-$N$ settings.
    \item \textbf{Weighted Regime:} Under covariate shift, applying BH directly to discrete weighted conformal $p$-values is not theoretically justified in general because their dependence can violate PRDS when weights are data-dependent. 
    Weighted Conformalized Selection (WCS) is designed to restore finite-sample FDR control in this setting and operates in two stages:
    \begin{enumerate}
        \item \textbf{Selection:} A preliminary rejection set is formed using a leave-one-out self-consistency check.
        \item \textbf{Pruning:} The set is reduced to control finite-sample FDR via three possible pruning strategies: \textit{Deterministic} (strict counting), \textit{Homogeneous} (shared randomization $\xi \sim U[0,1]$), and \textit{Heterogeneous} (individual randomization $\xi_j \sim U[0,1]$).
    \end{enumerate}
\end{itemize}
    
\textbf{Robustness:} Discrete weighted estimators are sensitive to the pruning method. As noted in \citet{Jin2025}, \textit{deterministic} pruning may yield low power due to coarse resolution, requiring \textit{homogeneous} or \textit{heterogeneous} randomization to smooth threshold effects. In contrast, our KDE-based surrogate $p$-values are continuous and do not cluster at the discrete mass points $w_j/W_{\text{total}}$. Consequently, the choice of WCS pruning strategy becomes asymptotically equivalent, as the probability of a $p$-value falling exactly on the rejection threshold is zero. We nevertheless retain the WCS wrapper with homogeneous pruning for all evaluated methods to ensure a uniform experimental pipeline.

\textit{Remark on Dependence and Guarantees:}
Weighted conformal $p$-values may violate the PRDS property, so applying BH directly is not covered by standard finite-sample theory. WCS attains finite-sample FDR control under $H_{\text{0}}$ (and suitable covariate-shift weights) by calibrating each test unit via auxiliary (leave-one-out) $p$-values \cite{Jin2025}. When we replace discrete weighted conformal $p$-values with our surrogate $p$-values, these guarantees do not automatically carry over: the KDE bandwidth and weight selection introduces global dependencies across calibration scores. Furthermore, when WCS is applied to our surrogate conformal $p$-values (without recomputing candidate-dependent auxiliary $p$-values), it often behaves similarly to the BH procedure on the same surrogate $p$-values. We assess the robustness of this approximation empirically in Section~\ref{sec:evaluation}.

\section{Theoretical Analysis}
\label{sec:theory}

Standard CAD provides \textit{finite-sample} marginal validity under exchangeability. In weighted settings, strictly maintaining this guarantee prompts a choice: the \textit{resolution floor} of the discrete estimator or the \textit{variance inflation} induced by randomization. We analyze these trade-offs, framing it as a trilemma among \textit{validity} (finite-sample guarantee), \textit{stability} (rejection consistency), and \textit{resolution} (rejection ability).

\subsection{Asymptotic Marginal Validity}

Let $Q_0$ denote the inlier distribution of the test points. For each test unit $j$, consider the null hypothesis $H_{0,j}:\; z_j \sim Q_0$, i.e., that test point $j$ is an inlier. A natural marginal validity target for a $p$-value $\hat p(z)$ is
\[
\mathbb{P}_{Z\sim Q_0}\bigl(\hat p(Z)\le u\bigr)\le u,\qquad \forall u\in[0,1],
\]
where the probability is over the calibration sample used to construct $\hat p$, and an independent null test point $Z\sim Q_0$. Let $F_0^w$ denote the (weighted) CDF of the null score $s(Z)$ under $Z\sim Q_0$. The ideal probability integral transform yields
\[
p^*(z) = 1 - F_0^w(s(z)),
\]
which is exactly $\mathrm{Unif}[0,1]$ under $H_{0,j}$ when $F_0^w$ is continuous. Our estimator $\hat p(z)$ approximates this target via estimated weights $\hat w$ and a weighted KDE.

\iffalse
\begin{theorem}[Consistency and asymptotic marginal validity of KDE $p$-values]\label{thm:asymp_valid}
Assume (i) $F_0^w$ is continuous; (ii) $\sup_z|\hat w(z)-w(z)|\xrightarrow{p}0$ and standard boundedness/moment conditions ensuring weighted KDE consistency; (iii) $K$ is bounded with $\int K=1$, $h_N\to0$, and $N h_N\to\infty$. Then
\[
\sup_t\bigl|\widehat F_0^w(t)-F_0^w(t)\bigr|\xrightarrow{p}0
\quad\text{and}\quad
\sup_z\bigl|\hat p(z)-p^*(z)\bigr|\xrightarrow{p}0.
\]
Moreover, for each fixed $u\in[0,1]$,
\[
\mathbb{P}_{Z\sim Q_0}\!\bigl(\hat p(Z)\le u\bigr)\le u+o(1),
\]
i.e., $\hat p$ is asymptotically (marginally) super-uniform.
\end{theorem}
\fi

\begin{theorem}[Consistency and asymptotic marginal validity of KDE $p$-values]\label{thm:asymp_valid}
Assume (i) $F_0^w$ is continuous; (ii) $\sup_z|\hat w(z)-w(z)|\xrightarrow{p}0$ and standard regularity conditions ensuring weighted KDE consistency; (iii) $K$ is bounded with $\int K=1$, the bandwidth $h_N\to0$, and \textbf{$N h_N/\log N\to\infty$}. Let $Z\sim Q_0$ be independent of the calibration data used to construct $\hat w$, $\widehat F_0^w$, and $\hat p$. Then
\[
\sup_t\bigl|\widehat F_0^w(t)-F_0^w(t)\bigr|\xrightarrow{p}0
\quad\text{and}\quad
\sup_z\bigl|\hat p(z)-p^*(z)\bigr|\xrightarrow{p}0.
\]
Moreover, for each fixed $u\in[0,1]$,
\[
\mathbb{P}_{Z\sim Q_0}\!\bigl(\hat p(Z)\le u\bigr)\le u+o(1),
\]
i.e., $\hat p$ is asymptotically (marginally) super-uniform.\footnote{A proof sketch is provided in Appendix~\ref{sec:proof_sketch}.}
\end{theorem}
\paragraph{Remark.}
KDE-based $p$-values $\hat p_{\mathrm{KDE}}$ are \emph{not} conformal and \emph{not} finite-sample valid. Conformal $p$-values satisfy finite-sample \textit{marginal} super-uniformity over the joint randomness of the calibration set and test point, but not conditional on a fixed calibration set. WCS achieves finite-sample FDR control with weighted conformal $p$-values satisfying its leave-one-out structure---substituting KDE surrogates does not preserve this guarantee. We therefore treat WCS-on-surrogates as a heuristic and evaluate it empirically, using smoothing and bandwidth selection to stabilize tail estimation and reduce discretization-induced Type II errors.

\subsection{Bias--Variance--Resolution Trilemma}

We compare deterministic $\hat p^{\mathrm{disc}}$ \eqref{eq:weighted_ecdf_cons}, randomized $\hat p^{\mathrm{rand}}$ \eqref{eq:weighted_ecdf_rand}, and KDE surrogate $\hat p_{\mathrm{KDE}}$ \eqref{eq:kde_pvalue} $p$-value construction.

\begin{enumerate}
\item \textbf{Deterministic (discrete).} $\hat p^{\mathrm{disc}}$ is deterministic but has a floor
$\min \hat p^{\mathrm{disc}} = w_j/(\sum_{i=1}^N w_i + w_j)$ (attained when $s_j>\max_i s_i$), which can exceed the multiple-testing cutoff and cause power collapse.

\item \textbf{Randomized.} $\hat p^{\mathrm{rand}}$ is conditionally uniform on an interval of width
$\mathrm{width}_j=(w_j+\sum_{i=1}^N w_i\mathbb I(s_i=s_j))/(\sum_{i=1}^N w_i+w_j)$ (via $U_j\sim\mathrm{Unif}[0,1]$). This removes the floor (unbounded resolution) but introduces conditional variance $\mathrm{Var}(\hat p^{\mathrm{rand}}\mid\cdot)=\mathrm{width}_j^2/12$, which is large when the relative test mass is large.

\item \textbf{Continuous (KDE).} $\hat p_{\mathrm{KDE}}$ replaces $U_j$-randomization by deterministic smoothing of the weighted calibration score distribution, yielding continuous $p$-values (no discrete floor) without auxiliary Monte Carlo variability, at the cost of replacing finite-sample exactness by asymptotic approximation (Theorem~1).
\end{enumerate}

\subsection{Formalizing the Dilemma}
We now formalize the limitations of the two standard weighted conformal $p$-values.

\begin{proposition}[Resolution Collapse through Discreteness]
Consider the deterministic weighted conformal $p$-value $p_j^{\mathrm{disc}}$ (Eq.~\ref{eq:weighted_ecdf_cons}). Conditional on the realized calibration scores $\{s_i\}_{i=1}^N$ and weights $\{w_i\}_{i=1}^N$, the smallest attainable value of the random variable $p_j^{\mathrm{disc}}$ over all possible test scores $s_j$ occurs when $s_j>\max_i s_i$, and equals
\begin{equation}
\min_{s_j\in\mathbb R} p_j^{\mathrm{disc}}(s_j)
=
\frac{w_j}{W_{\mathrm{total}}}.
\end{equation}
\end{proposition}

\textbf{Implication.} If $w_j/W_{\mathrm{total}}$ exceeds the multiple-testing threshold even arbitrarily extreme test scores cannot be rejected using the deterministic weighted conformal $p$-value.

\begin{proposition}[Conditional Variance through Randomness]
Let $p_j^{\mathrm{rand}}$ be the randomized weighted conformal $p$-value (Eq.~\ref{eq:weighted_ecdf_rand}). Conditional on the realized calibration scores, weights and $s_j$, the randomness in $p_j^{\mathrm{rand}}$ comes from $U_j\sim\mathrm{Unif}(0,1)$ and yields an interval of width $\mathrm{width}_j$. Consequently,
\begin{equation}
\mathrm{Var}\!\left(p_j^{\mathrm{rand}}
  \mid \{s_i,w_i\}_{i=1}^N, s_j\right)
=
\frac{1}{12}
\left(
\frac{w_j + W_{=}(s_j)}
     {W_{\mathrm{total}}}
\right)^2.
\end{equation}
\end{proposition}

\begin{proof}
Conditional on $\{s_i,w_i\}_{i=1}^N$ and $s_j$, the randomized $p$-value has the form $p_j^{\mathrm{rand}} = a_j + U_j\cdot \mathrm{width}_j$
for some deterministic $a_j$ (the left endpoint of the interval). Since $\mathrm{Var}(U_j)=1/12$, the stated variance follows.
\end{proof}

\textbf{Implication:} The variance of the $p$-value scales with the test weight fraction. If $N_{\text{eff}}$ is low ($w_j$ is large), the $p$-value becomes a noisy estimate. This noise acts as a regularizer, potentially masking anomalies and reducing statistical power.

\section{Evaluation}\label{sec:evaluation}

We compare our proposed method against standard conformal procedures by a two-phase experimental protocol. All experiments are conducted on standard anomaly detection benchmark datasets \cite{Han2022} (see Table~\ref{tab:data}). The anomaly rate of each test set is controlled at $\pi \approx 0.05$. All features are z-score standardized, with parameters fitted on the training splits for each randomized trial.

\subsection{Phase 1: Model Selection}
CAD requires a suitable scoring function to produce informative $p$-values. To this end, we employ model selection for each experimental trial (so \textit{per random seed}).

For each of the $N_{\text{seeds}}=20$ trials:
\begin{enumerate}
    \item We randomly partition the available data into a \textbf{Training Set} ($D_{\text{train}}$), a \textbf{Validation Set} ($D_{\text{val}}$), and a \textbf{Test Set} ($D_{\text{test}}$).
    \item We train candidate detectors (see Table~\ref{tab:models}) on $D_{\text{train}}$ with default hyperparameters, using \textit{Jackknife+-after-Bootstrap} calibration \cite{Kim2020,Hennhoefer2024} that integrates training and calibration by bootstrap sampling to make more efficient use of the data without requiring a disjoint $D_{\text{calib}}$ via splits.
    \item We evaluate these candidates on $D_{\text{val}}$ (containing both inliers and anomalies). A model is then selected based on a lexicographical hierarchy: maximizing PR-AUC, then ROC-AUC, then minimizing Brier Score.
    \item The selected model is fixed for this trial for \textit{all} evaluated methods. Crucially, $D_{\text{val}}$ is then discarded and \textit{not} used for further evaluation to prevent \textit{data leakage}.
\end{enumerate}

\subsection{Phase 2: Model Inference}

Using the fixed seed--model pairs from Phase~1, we compare the proposed method with four standard approaches (unweighted and weighted, each in deterministic and randomized variants). 

Specifically, we consider: (i) the deterministic EDF-based procedure, which isolates the effect of resolution collapse, (ii) the randomized variant, which isolates the effect of variance inflation and (iii) our proposed continuous relaxation leveraging both unweighted and weighted KDE.

Weighted procedures employ WCS with \textit{homogeneous} pruning while unweighted procedures employ the BH method. All weights are estimated by a probabilistic Random Forest classifier\footnote{Estimates are stabilized via bagging; see Appendix~\ref{sec:stable_weight_est}.}. The nominal FDR is controlled at $\alpha = 0.1$.

\subsection{Evaluation Metrics}
\label{sec:metrics}
For a test batch, let $H_{\text{0}}$ denote the set of true inliers and $H_{\text{1}}$ the set of true anomalies. Let $\mathcal{R}$ be the set of indices rejected. We evaluate performance using two metrics:

\paragraph{False Discovery Proportion (FDP).}
The empirical fraction of false alarms among the reported discoveries:
\begin{equation}
    \mathrm{FDP} = \frac{|\{i \in \mathcal{R} : i \in \mathcal{H}_0\}|}{\max(1, |\mathcal{R}|)}.
\end{equation}
We report the marginal FDR, estimated by averaging the FDP over all $N_{\text{seeds}}$ trials, as
$\widehat{\mathrm{FDR}} \approx \mathbb{E}[\mathrm{FDP}]$. A method is considered \textit{valid} if $\widehat{\mathrm{FDR}} \le \alpha + t_{0.995} \times \frac{\sigma(\widehat{\mathrm{FDR}})}{\sqrt{20}}$.

\paragraph{Statistical Power.}
The proportion of true anomalies correctly identified:
\begin{equation}
    \mathrm{Power} = \frac{|\{i \in \mathcal{R} : i \in \mathcal{H}_1\}|}{|\mathcal{H}_1|}.
\end{equation}
We report the mean $\mathrm{Power}$ over the same $N_{\text{seeds}}$ trials.

\begin{table*}[tp]
\centering
\caption{\textbf{Performance of conformal inference strategies across anomaly detection benchmarks.}
All weighted methods employ WCS with \textit{\ul{homogeneous} pruning} to guarantee finite-sample FDR control.
Values represent the mean $\pm$ standard deviation of the empirical marginal FDR and statistical power aggregated over 20 independent trials with randomized splits.
We compare deterministic and randomized baselines against the proposed continuous (KDE-based) approach.
The calibration and test set sizes are denoted by $n_{\text{train}}$ and $n_{\text{test}}$. Underlined values indicate validity violations (see Section~\ref{sec:metrics}).
}
\label{tab:results_table}
\renewcommand{\arraystretch}{1.15}
\setlength{\tabcolsep}{4.5pt}
\begin{tabularx}{\textwidth}{l l c c c c p{0.5em} C C}
\toprule
\textbf{Dataset} & \textbf{Method}
& \multicolumn{2}{c}{\textbf{Deterministic}}
& \multicolumn{2}{c}{\textbf{Randomized}}
&
& $n_{\text{train}}$
& $n_{\text{test}}$\\
\cmidrule(lr){3-4}\cmidrule(lr){5-6}
& \textit{Homogeneous} & \textbf{FDR} & \textbf{Power} & \textbf{FDR} & \textbf{Power} & & \\
\midrule

\multirow{4}{*}{WBC}
& EDF                   & $0.000\pm0.000$ & $0.000\pm0.000$ & $0.062\pm0.129$ & $0.200\pm0.332$ & & \multirow{4}{*}{106} & \multirow{4}{*}{56} \\
& Weighted EDF          & $0.000\pm0.000$ & $0.000\pm0.000$ & $0.075\pm0.245$ & $0.100\pm0.244$ & &  &   \\
& \textbf{KDE}          & $0.095\pm0.152$ & $0.500\pm0.351$ & \multicolumn{2}{c}{\textemdash} & &  &   \\
& \textbf{Weighted KDE} & $0.078\pm0.142$ & $0.417\pm0.373$ & \multicolumn{2}{c}{\textemdash} & &  &   \\
\cmidrule(lr){1-9}

\multirow{4}{*}{Ionosphere}
& EDF                   & $0.000\pm0.000$ & $0.000\pm0.000$ & $0.076\pm0.199$ & $0.150\pm0.221$ & & \multirow{4}{*}{112} & \multirow{4}{*}{88} \\
& Weighted EDF          & $0.000\pm0.000$ & $0.000\pm0.000$ & $0.042\pm0.131$ & $0.075\pm0.143$ & &  &   \\
& \textbf{KDE}          & $0.081\pm0.150$ & $0.300\pm0.434$ & \multicolumn{2}{c}{\textemdash} & & &   \\
& \textbf{Weighted KDE} & $0.047\pm0.119$ & $0.138\pm0.339$ & \multicolumn{2}{c}{\textemdash} & &  &   \\
\cmidrule(lr){1-9}

\multirow{4}{*}{WDBC}
& EDF                   & $0.098\pm0.165$ & $0.280\pm0.442$ & \ul{$0.166\pm0.197$} & $0.440\pm0.452$ & & \multirow{4}{*}{178} & \multirow{4}{*}{92} \\
& Weighted EDF          & $0.000\pm0.000$ & $0.000\pm0.000$ & $0.108\pm0.197$ & $0.090\pm0.152$ &  &  &  \\
& \textbf{KDE}          & $0.086\pm0.135$ & $0.390\pm0.402$ & \multicolumn{2}{c}{\textemdash} &  &  &  \\
& \textbf{Weighted KDE} & $0.095\pm0.164$ & $0.350\pm0.383$ & \multicolumn{2}{c}{\textemdash} &  &  &  \\
\cmidrule(lr){1-9}

\multirow{4}{*}{\shortstack{Breast Cancer\\(Wisconsin)}}
& EDF                   & $0.000\pm0.000$ & $0.000\pm0.000$ & $0.000\pm0.000$ & $0.094\pm0.145$ & & \multirow{4}{*}{222} & \multirow{4}{*}{171} \\
& Weighted EDF          & $0.000\pm0.000$ & $0.000\pm0.000$ & $0.000\pm0.000$ & $0.044\pm0.084$ &  &  &  \\
& \textbf{KDE}          & $0.046\pm0.074$ & $0.350\pm0.296$ & \multicolumn{2}{c}{\textemdash} &  &  &  \\
& \textbf{Weighted KDE} & $0.027\pm0.066$ & $0.267\pm0.218$ & \multicolumn{2}{c}{\textemdash} &  &  &  \\
\cmidrule(lr){1-9}

\multirow{4}{*}{Vowels}
& EDF                   & $0.000\pm0.000$ & $0.000\pm0.000$ & $0.017\pm0.074$ & $0.067\pm0.071$ & & \multirow{4}{*}{703} & \multirow{4}{*}{364} \\
& Weighted EDF          & $0.000\pm0.000$ & $0.000\pm0.000$ & $0.000\pm0.000$ & $0.014\pm0.040$ &  &  &  \\
& \textbf{KDE}          & $0.035\pm0.109$ & $0.122\pm0.082$ & \multicolumn{2}{c}{\textemdash} &  &  &  \\
& \textbf{Weighted KDE} & $0.035\pm0.109$ & $0.117\pm0.082$ & \multicolumn{2}{c}{\textemdash} &  &  &  \\
\cmidrule(lr){1-9}

\multirow{4}{*}{Cardio}
& EDF                   & $0.034\pm0.091$ & $0.039\pm0.097$ & $0.119\pm0.250$ & $0.067\pm0.096$ & & \multirow{4}{*}{827} & \multirow{4}{*}{458} \\
& Weighted EDF          & $0.018\pm0.081$ & $0.015\pm0.068$ & $0.031\pm0.072$ & $0.030\pm0.072$ &  & &  \\
& \textbf{KDE}          & $0.066\pm0.149$ & $0.089\pm0.079$ & \multicolumn{2}{c}{\textemdash} &  &  & \\
& \textbf{Weighted KDE} & $0.035\pm0.098$ & $0.087\pm0.080$ & \multicolumn{2}{c}{\textemdash} &  &  &  \\
\cmidrule(lr){1-9}

\multirow{4}{*}{Musk}
& EDF                   & $0.102\pm0.060$ & $1.000\pm0.000$  & \ul{$0.105\pm0.060$} & $1.000\pm0.000$ & & \multirow{4}{*}{1,482} & \multirow{4}{*}{766} \\
& Weighted EDF          & $0.096\pm0.056$ & $1.000\pm0.000$   & $0.103\pm0.056$ & $1.000\pm0.000$ &  &  &  \\
& \textbf{KDE}          & $0.084\pm0.060$ & $1.000\pm0.000$  & \multicolumn{2}{c}{\textemdash} &  &  &  \\
& \textbf{Weighted KDE} & $0.082\pm0.060$ & $1.000\pm0.000$  & \multicolumn{2}{c}{\textemdash} &  &  &  \\
\cmidrule(lr){1-9}

\multirow{4}{*}{Satellite}
& EDF                   & \ul{$0.109\pm0.107$} & $0.259\pm0.082$ & \ul{$0.108\pm0.105$} & $0.267\pm0.085$ & & \multirow{4}{*}{2,199} & \multirow{4}{*}{1,609} \\
& Weighted EDF          & $0.104\pm0.099$ & $0.249\pm0.085$ & \ul{$0.107\pm0.103$} & $0.259\pm0.083$ &  &  &  \\
& \textbf{KDE}          & \ul{$0.117\pm0.103$} & $0.291\pm0.087$ & \multicolumn{2}{c}{\textemdash} &  &  &  \\
& \textbf{Weighted KDE} & \ul{$0.112\pm0.098$} & $0.284\pm0.070$ & \multicolumn{2}{c}{\textemdash} &  &  &  \\
\cmidrule(lr){1-9}

\multirow{4}{*}{Mammography}
& EDF                   & $0.019\pm0.037$ & $0.052\pm0.057$ & $0.026\pm0.040$ & $0.069\pm0.053$ & & \multirow{4}{*}{5,461} & \multirow{4}{*}{2,796} \\
& Weighted EDF          & $0.017\pm0.037$ & $0.038\pm0.049$ & $0.017\pm0.037$ & $0.043\pm0.046$ &  &  &  \\
& \textbf{KDE}          & $0.045\pm0.060$ & $0.085\pm0.050$ & \multicolumn{2}{c}{\textemdash} &  &  &  \\
& \textbf{Weighted KDE} & $0.037\pm0.051$ & $0.077\pm0.048$ & \multicolumn{2}{c}{\textemdash} &  &  & \\
\bottomrule
\end{tabularx}
\end{table*}

\begin{figure*}[t]
  \centering
  \begin{tikzpicture}

\definecolor{colWEDF}{HTML}{72B7B2}
\definecolor{colWKDE}{HTML}{B279A2}

\definecolor{colEDFdet}{HTML}{E4F1E1}  % teal
\definecolor{colEDFhom}{HTML}{63A6a0}  % original teal (center)
\definecolor{colEDFhet}{HTML}{0d585f}  % greener teal

% Styles (define outside axis = safest)
\pgfplotsset{
  edfdet/.style={
    draw=black,
    line width=0.6pt,
    fill=colEDFdet,
    fill opacity=1,
  },
  edfhomog/.style={
    draw=black,
    line width=0.6pt,
    fill=colEDFhom,
    fill opacity=1,
  },
  edfhet/.style={
    draw=black,
    line width=0.6pt,
    fill=colEDFhet,
    fill opacity=1,
  },
  kdestd/.style={
    draw=black,
    line width=0.6pt,
    fill=colWKDE,
    fill opacity=1,
  }
}

\begin{axis}[
    width=\textwidth,
    height=0.32\textwidth,
    ybar,
    bar width=6.6pt,
    ymin=0, ymax=0.52,
    ymajorgrids=true,
    tick align=outside,
    ylabel={Statistical Power},
    xtick pos=left,   % x ticks only at the bottom
    ytick pos=left,   % y ticks only at the left
    grid style={line width=0.5pt, draw=black!15},
    axis line style={line width=0.8pt},
    tick style={line width=0.8pt},
    label style={font=\normalsize},
    tick label style={font=\small},
    x tick label style={align=center},
    symbolic x coords={
        WBC, Ionosphere, WDBC, BreastCa,
        Vowels, Cardio, Satellite, Mammogr.
    },
    xtick=data,
    enlarge x limits=0.07,
    clip=true,
    legend style={
        at={(0.5,0.85)},
        anchor=south,
        draw=black,
        fill=white,
        fill opacity=0.95,
        text opacity=1,
        rounded corners=1pt,
        inner sep=2.5pt,
        font=\scriptsize,
        /tikz/every even column/.append style={column sep=6pt},
    },
    legend columns=4,
]

% W. EDF (det.)
\addplot+[
  edfdet,
  error bars/.cd,
    y dir=both,
    y explicit,
    error bar style={line width=0.8pt, black},
    error mark options={rotate=90, mark size=2.2pt, line width=0.8pt, black},
] coordinates {
    (WBC,0.017) +- (0,0.0166)     
    (Ionosphere,0.025) +- (0,0.0172)
    (WDBC,0.010) +- (0,0.0100)
    (BreastCa,0.006) +- (0,0.0056)
    (Vowels,0.000) +- (0,0.0000)
    (Cardio,0.024) +- (0,0.0153)
    (Satellite,0.233) +- (0,0.0253)
    (Mammogr.,0.039) +- (0,0.0108)
};
\addlegendentry{W.\ EDF (deterministic)}

% W. EDF (homog.)
\addplot+[
  edfhomog,
  error bars/.cd,
    y dir=both,
    y explicit,
    error bar style={line width=0.8pt, black},
    error mark options={rotate=90, mark size=2.2pt, line width=0.8pt, black},
] coordinates {
    (WBC,0.1) +- (0,0.0546)
    (Ionosphere,0.075) +- (0,0.0329)
    (WDBC,0.090) +- (0,0.0340)
    (BreastCa,0.044) +- (0,0.0187)
    (Vowels,0.014) +- (0,0.0089)
    (Cardio,0.030) +- (0,0.0161)
    (Satellite,0.259) +- (0,0.0186)
    (Mammogr.,0.043) +- (0,0.0103)
};
\addlegendentry{W.\ EDF (homogeneous)}

% W. EDF (het.)
\addplot+[
  edfhet,
  error bars/.cd,
    y dir=both,
    y explicit,
    error bar style={line width=0.8pt, black},
    error mark options={rotate=90, mark size=2.2pt, line width=0.8pt, black},
] coordinates {
    (WBC,0.083) +- (0,0.0410)
    (Ionosphere,0.075) +- (0,0.0319)
    (WDBC,0.100) +- (0,0.0271)
    (BreastCa,0.044) +- (0,0.0187)
    (Vowels,0.011) +- (0,0.0065)
    (Cardio,0.028) +- (0,0.0155)
    (Satellite,0.261) +- (0,0.0182)
    (Mammogr.,0.045) +- (0,0.0101)
};
\addlegendentry{W.\ EDF (heterogeneous)}

\addplot+[
  kdestd,
  error bars/.cd,
    y dir=both,
    y explicit,
    error bar style={line width=0.8pt, black},
    error mark options={rotate=90, mark size=2.2pt, line width=0.8pt, black},
] coordinates {
    (WBC,0.417) +- (0,0.0833)
    (Ionosphere,0.138) +- (0,0.0758)
    (WDBC,0.350) +- (0,0.0857)
    (BreastCa,0.267) +- (0,0.0486)
    (Vowels,0.117) +- (0,0.0184)
    (Cardio,0.087) +- (0,0.0179)
    (Satellite,0.284) +- (0,0.0157)
    (Mammogr.,0.077) +- (0,0.0108)
};
\addlegendentry{W.\ KDE}

\end{axis}

\end{tikzpicture}
  \caption{\textbf{Statistical power of weighted and randomized conformal methods with WCS pruning strategies, including the weighted KDE-based approach.} Labels refer to the WCS pruning method. Results for \textit{Musk} are omitted due to ceiling performance across all strategies. Error bars denote mean $\pm$ standard error over 20 randomized trials.}

\label{fig:results_plot}
\end{figure*}
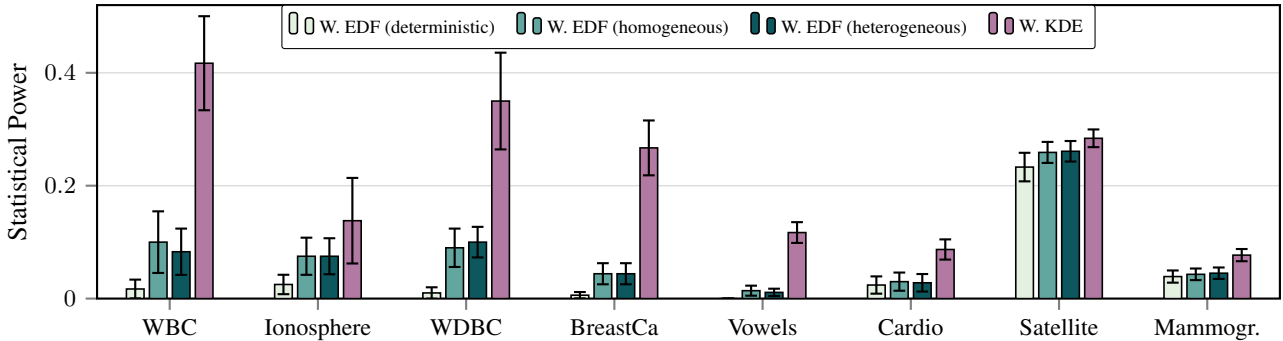

\section{Results}\label{sec:results}

Table~\ref{tab:results_table} and Figure~\ref{fig:results_plot} show the evaluation results based on the described protocol. The results provide empirical validation of the theoretical dilemmas posed in Section~\ref{sec:theory}, demonstrating that continuous inference relaxation is often a prerequisite for weighted CAD in low-data regimes.

\subsection{The Inability to reject}
Smaller datasets ($N \ll 1000$) illustrate the inflation of the minimum attainable $p$-values of discrete estimators.
\begin{itemize}
    \item \textbf{Failure through Discreteness:} Both standard discrete methods yield few to no discoveries. The sample size is insufficient to generate a $p$-value below the detection threshold. Randomization recovers some of the lost ability to make discoveries.
    \item \textbf{Recovery through Continuity:} The KDE approach successfully extrapolates the tail behaviour, recovering significant statistical power while maintaining valid marginal FDR control.
\end{itemize}

Larger calibration sets ($N>1000$) reduce kernel-induced bias, so all methods converge in performance as $N \to \infty$.

\subsection{The Cost of Importance Weighting}
The dataset \textit{WDBC} clearly demonstrates how covariate shift adaptation via importance weighting affects $N_{\text{eff}}$.

\begin{itemize}
    \item \textbf{Uniform Weighting:} For both standard (unweighted) conformal methods $D_{\text{calib}}$ is sufficient to achieve discoveries. However, note the validity violation of the unweighted, randomized approach in Table~\ref{tab:results_table}.
    \item \textbf{Importance Weighting:} After adaption to covariate shift, higher weighted calibration instances dominate the mass, decreasing $N_{\text{eff}}$ and leading to a severe loss in statistical power for the EDF-based weighted methods.
\end{itemize}

The weighted KDE-based approach maintains its statistical power by decoupling its ability to reject from $N_{\text{eff}}$.

\subsection{FDR Control and Validity}
With one mild violation, the KDE-based approaches maintain valid marginal FDR control. In our experiments, the asymptotic approximation underlying the KDE does not systematically affect error control. Figure~\ref{fig:results_plot} illustrates the impact of the pruning method on the randomized variants, which are less powerful than the KDE-based approach.

\section{Conclusion}\label{sec:conclusion}

We formalized the \textit{Resolution--Variance Dilemma} in weighted CAD. As importance weights of calibration and test instances localize under covariate shift adaption, the effective sample size of the calibration set decreases, trapping inference in low-data regimes between discrete EDFs that become incapable of rejection in multiple testing (\textit{resolution collapse}) and a randomized variant that regains continuity by noise injection, degrading the signal (\textit{variance inflation}).

To break this trade-off, we proposed a continuous weighted inference scheme via weighted KDE. Empirically, it (1) restores detection in data-scarce regimes where discrete baselines yield zero discoveries, (2) improves efficiency over stochastic smoothing by replacing it with deterministic smoothing, and (3) matches discrete performance as calibration data grows, consistent with asymptotic convergence.

Conceptually, our framework is best viewed as a pragmatic extension of rigorous CAD: finite-sample exactness is theoretically strongest, but becomes operationally vacuous when discreteness prevents rejections. By accepting an \textit{asymptotic} validity target, continuous smoothing decouples rejection ability from calibration set size and extends CAD to settings where finite-sample exactness guarantees under (weighted) exchangeability yield no utility.

\begin{acknowledgements} % will be removed in pdf for initial submission,
						 % (without ‘accepted’ option in \documentclass)
                         % so you can already fill it to test with the
                         % ‘accepted’ class option
    This work was conducted as part of the research project \textit{Biflex Industrie} (grant number 01MV23020A), funded by the German Federal Ministry for Economic Affairs and Climate Action (BMWK).
\end{acknowledgements}

% References
\bibliography{references}

\newpage
\onecolumn

\title{Between Resolution Collapse and Variance Inflation:\\
Weighted Conformal Anomaly Detection in Low-Data Regimes\\
(Supplementary Material)}
\maketitle

\appendix
\section{Implementation and Reproducibility Details}

All experiments reported in this paper can be reproduced end-to-end from code provided on \href{https://github.com/OliverHennhoefer/wkde-cad}{\texttt{https://github.com/OliverHennhoefer/wkde-cad}}. This environment fully specifies the Python version, and all required dependencies. Executing the provided configuration files and experiment scripts reproduces the complete experimental pipeline, including data preprocessing, model selection, model training and evaluation procedures as presented in the main text and supplementary material.

All conformal methods as well as the method proposed in this work are implemented in the publicly available Python package \href{https://github.com/OliverHennhoefer/nonconform}{\texttt{nonconform}}, available on \texttt{PyPI}. The implementations are compatible with standard \texttt{scikit-learn} interfaces, as well as \texttt{pyod} and custom detector classes, and operate on data represented as either \texttt{numpy} arrays or \texttt{pandas} data frames. 

The package itself focuses exclusively on providing reusable method implementations for personal use. It does not include the experimental protocol, benchmarking framework, or evaluation pipeline used to produce the results in this paper. These components are provided separately within the reproducibility environment described above. This separation ensures that the methodological contributions can be applied independently to user-provided datasets while maintaining full reproducibility of the reported empirical results.

\newpage

\section{Evaluation}

\begin{center}
\captionof{table}{\textbf{Overview of the models used for evaluation, their abbreviations, categories, and references.}}
\label{tab:models}
\begin{tabularx}{\textwidth}{@{}l l X@{}}
\toprule
\textbf{Model} & \textbf{Category} & \textbf{Reference} \\
\midrule
Isolation Forest (\textbf{IForest}) & Tree-based & \cite{Liu2008} \\
Lightweight Online Detector of Anomalies (\textbf{LODA}) & Projection-based & \cite{Pevny2016} \\
Isolation Nearest Neighbor Ensemble (\textbf{INNE}) & Neighbor-based & \cite{Bandaragoda2018} \\
Histogram-Based Outlier Score (\textbf{HBOS}) & Density/Distance-based & \cite{Goldstein2012} \\
Copula-Based OD (\textbf{COPOD}) & Copula-based & \cite{Li2020} \\
Empirical Cumulative Distribution OD (\textbf{ECOD}) & Distribution-based & \cite{Li2023} \\
\bottomrule
\end{tabularx}
\end{center}

\begin{center}
\captionof{table}{\textbf{Overview of datasets used for evaluation, their categories, and references.}}
\label{tab:data}
\begin{tabularx}{\textwidth}{XXX}
\toprule
\textbf{Dataset} & \textbf{Category} & \textbf{Reference} \\
\midrule
WBC & Healthcare & \cite{Mangasarian1995} \\
Ionosphere & Oryctognosy & \cite{Sigillito1989} \\
WDBC & Healthcare & \cite{Mangasarian1995} \\
Breast Cancer & Healthcare & \cite{Wolberg1990} \\
Vowels & Linguistics & \cite{Kudo1999} \\
Cardio & Healthcare & \cite{Campos2000} \\
Musk & Chemistry & \cite{Dietterich1993} \\
Satellite & Astronautics & \cite{Rayana2016} \\
Mammography & Healthcare & \cite{Woods1993} \\
\bottomrule
\end{tabularx}
\end{center}

% Preamble (if needed):
% \usepackage{booktabs}
% \usepackage{tabularx}
% \usepackage{ragged2e}
% \usepackage{multirow}
% \newcolumntype{Y}{>{\RaggedRight\arraybackslash}X}

% -------------------- Table 1 (4 datasets) --------------------
\begin{table*}[t]
\centering
\caption{\textbf{Per-dataset detector comparison (mean$\pm$std over 20 seeds).}
Results are aggregated across independent trials with per-seed model selection.
\textbf{Wins} counts how often (out of 20) a detector is selected for the dataset according to the lexicographic rule PR-AUC $\uparrow$, ROC-AUC $\uparrow$, Brier $\downarrow$.}
\label{tab:benchmark_results_part1}
\setlength{\tabcolsep}{5pt}
\renewcommand{\arraystretch}{1.10}

\begin{tabularx}{\textwidth}{l l CCC l}
\toprule
\textbf{Dataset} & \textbf{Model} & \textbf{PR-AUC (mean $\pm$ std)} & \textbf{ROC-AUC (mean $\pm$ std)} & \textbf{Brier (mean $\pm$ std)} & \textbf{Wins} \\
\midrule
\multirow{6}{*}{WBC}
& \textbf{IForest} & 0.987 $\pm$ 0.009 & 0.993 $\pm$ 0.004 & 0.051 $\pm$ 0.005 & 12/20 \\
& \textbf{COPOD}   & 0.984 $\pm$ 0.008 & 0.992 $\pm$ 0.004 & 0.048 $\pm$ 0.011 & 3/20 \\
& \textbf{ECOD}    & 0.984 $\pm$ 0.008 & 0.992 $\pm$ 0.004 & 0.049 $\pm$ 0.012 & 2/20 \\
& \textbf{HBOS}    & 0.972 $\pm$ 0.014 & 0.988 $\pm$ 0.006 & 0.050 $\pm$ 0.006 & 3/20 \\
& LODA    & 0.913 $\pm$ 0.029 & 0.962 $\pm$ 0.011 & 0.071 $\pm$ 0.009 & 0/20 \\
& INNE    & 0.908 $\pm$ 0.042 & 0.954 $\pm$ 0.019 & 0.219 $\pm$ 0.011 & 0/20 \\
\midrule
\multirow{6}{*}{Ionosphere}
& \textbf{INNE}    & 0.952 $\pm$ 0.020 & 0.960 $\pm$ 0.017 & 0.214 $\pm$ 0.018 & 20/20 \\
& IForest & 0.859 $\pm$ 0.034 & 0.898 $\pm$ 0.027 & 0.142 $\pm$ 0.012 & 0/20 \\
& ECOD    & 0.775 $\pm$ 0.032 & 0.825 $\pm$ 0.028 & 0.158 $\pm$ 0.011 & 0/20 \\
& LODA    & 0.757 $\pm$ 0.037 & 0.841 $\pm$ 0.022 & 0.166 $\pm$ 0.010 & 0/20 \\
& COPOD   & 0.756 $\pm$ 0.033 & 0.833 $\pm$ 0.024 & 0.162 $\pm$ 0.013 & 0/20 \\
& HBOS    & 0.534 $\pm$ 0.051 & 0.685 $\pm$ 0.045 & 0.271 $\pm$ 0.019 & 0/20 \\
\midrule
\multirow{6}{*}{WDBC}
& \textbf{COPOD}   & 0.985 $\pm$ 0.012 & 0.996 $\pm$ 0.003 & 0.053 $\pm$ 0.004 & 11/20 \\
& \textbf{INNE}    & 0.976 $\pm$ 0.019 & 0.994 $\pm$ 0.005 & 0.223 $\pm$ 0.012 & 8/20 \\
& \textbf{IForest} & 0.973 $\pm$ 0.016 & 0.992 $\pm$ 0.004 & 0.070 $\pm$ 0.008 & 1/20 \\
& HBOS    & 0.959 $\pm$ 0.024 & 0.988 $\pm$ 0.006 & 0.090 $\pm$ 0.009 & 0/20 \\
& ECOD    & 0.922 $\pm$ 0.027 & 0.975 $\pm$ 0.008 & 0.090 $\pm$ 0.007 & 0/20 \\
& LODA    & 0.898 $\pm$ 0.029 & 0.974 $\pm$ 0.008 & 0.066 $\pm$ 0.009 & 0/20 \\
\midrule
\multirow{6}{*}{\shortstack{Breast Cancer\\(Wisconsin)}}
& \textbf{IForest} & 0.990 $\pm$ 0.004 & 0.995 $\pm$ 0.002 & 0.041 $\pm$ 0.004 & 13/20 \\
& \textbf{COPOD}   & 0.987 $\pm$ 0.004 & 0.994 $\pm$ 0.002 & 0.047 $\pm$ 0.003 & 5/20 \\
& \textbf{ECOD}    & 0.987 $\pm$ 0.004 & 0.994 $\pm$ 0.002 & 0.049 $\pm$ 0.003 & 2/20 \\
& HBOS    & 0.978 $\pm$ 0.008 & 0.991 $\pm$ 0.003 & 0.043 $\pm$ 0.003 & 0/20 \\
& LODA    & 0.950 $\pm$ 0.016 & 0.983 $\pm$ 0.005 & 0.059 $\pm$ 0.006 & 0/20 \\
& INNE    & 0.939 $\pm$ 0.017 & 0.974 $\pm$ 0.006 & 0.237 $\pm$ 0.010 & 0/20 \\
\bottomrule
\end{tabularx}
\end{table*}

% -------------------- Table 2 (5 datasets) --------------------
\begin{table*}[t]
\centering
\caption{\textbf{Per-dataset detector comparison (mean$\pm$std over 20 seeds).}
Results are aggregated across independent trials with per-seed model selection.
\textbf{Wins} counts how often (out of 20) a detector is selected for the dataset according to the lexicographic rule PR-AUC $\uparrow$, ROC-AUC $\uparrow$, Brier $\downarrow$.}
\label{tab:benchmark_results_part2}
\setlength{\tabcolsep}{5pt}
\renewcommand{\arraystretch}{1.10}

\begin{tabularx}{\textwidth}{l l CCC l}
\toprule
\textbf{Dataset} & \textbf{Model} & \textbf{PR-AUC (mean $\pm$ std)} & \textbf{ROC-AUC (mean $\pm$ std)} & \textbf{Brier (mean $\pm$ std)} & \textbf{Wins} \\
\midrule
\multirow{6}{*}{Vowels}
& \textbf{INNE}    & 0.594 $\pm$ 0.092 & 0.913 $\pm$ 0.028 & 0.176 $\pm$ 0.012 & 20/20 \\
& IForest & 0.312 $\pm$ 0.087 & 0.783 $\pm$ 0.049 & 0.188 $\pm$ 0.013 & 0/20 \\
& LODA    & 0.243 $\pm$ 0.065 & 0.696 $\pm$ 0.057 & 0.204 $\pm$ 0.012 & 0/20 \\
& HBOS    & 0.240 $\pm$ 0.076 & 0.701 $\pm$ 0.053 & 0.236 $\pm$ 0.015 & 0/20 \\
& ECOD    & 0.227 $\pm$ 0.070 & 0.628 $\pm$ 0.066 & 0.228 $\pm$ 0.017 & 0/20 \\
& COPOD   & 0.114 $\pm$ 0.028 & 0.526 $\pm$ 0.061 & 0.221 $\pm$ 0.010 & 0/20 \\
\midrule
\multirow{6}{*}{Cardio}
& \textbf{ECOD}    & 0.752 $\pm$ 0.039 & 0.962 $\pm$ 0.007 & 0.089 $\pm$ 0.006 & 18/20 \\
& \textbf{INNE}    & 0.707 $\pm$ 0.038 & 0.955 $\pm$ 0.009 & 0.249 $\pm$ 0.006 & 2/20 \\
& IForest & 0.692 $\pm$ 0.044 & 0.947 $\pm$ 0.011 & 0.097 $\pm$ 0.008 & 0/20 \\
& COPOD   & 0.654 $\pm$ 0.047 & 0.939 $\pm$ 0.010 & 0.090 $\pm$ 0.012 & 0/20 \\
& LODA    & 0.620 $\pm$ 0.059 & 0.920 $\pm$ 0.017 & 0.077 $\pm$ 0.009 & 0/20 \\
& HBOS    & 0.522 $\pm$ 0.053 & 0.840 $\pm$ 0.024 & 0.134 $\pm$ 0.012 & 0/20 \\
\midrule
\multirow{6}{*}{Musk}
& \textbf{HBOS}    & 1.000 $\pm$ 0.000 & 1.000 $\pm$ 0.000 & 0.070 $\pm$ 0.005 & 20/20 \\
& ECOD    & 1.000 $\pm$ 0.000 & 1.000 $\pm$ 0.000 & 0.088 $\pm$ 0.005 & 0/20 \\
& INNE    & 1.000 $\pm$ 0.000 & 1.000 $\pm$ 0.000 & 0.146 $\pm$ 0.004 & 0/20 \\
& LODA    & 0.959 $\pm$ 0.025 & 0.998 $\pm$ 0.001 & 0.082 $\pm$ 0.009 & 0/20 \\
& COPOD   & 0.497 $\pm$ 0.044 & 0.959 $\pm$ 0.006 & 0.193 $\pm$ 0.006 & 0/20 \\
& IForest & 0.449 $\pm$ 0.157 & 0.944 $\pm$ 0.027 & 0.221 $\pm$ 0.012 & 0/20 \\
\midrule
\multirow{6}{*}{Satellite}
& \textbf{HBOS}    & 0.815 $\pm$ 0.012 & 0.866 $\pm$ 0.009 & 0.135 $\pm$ 0.003 & 20/20 \\
& INNE    & 0.788 $\pm$ 0.011 & 0.839 $\pm$ 0.012 & 0.166 $\pm$ 0.003 & 0/20 \\
& IForest & 0.773 $\pm$ 0.016 & 0.804 $\pm$ 0.018 & 0.139 $\pm$ 0.005 & 0/20 \\
& LODA    & 0.716 $\pm$ 0.018 & 0.698 $\pm$ 0.019 & 0.166 $\pm$ 0.005 & 0/20 \\
& COPOD   & 0.679 $\pm$ 0.016 & 0.695 $\pm$ 0.017 & 0.181 $\pm$ 0.005 & 0/20 \\
& ECOD    & 0.645 $\pm$ 0.016 & 0.650 $\pm$ 0.019 & 0.190 $\pm$ 0.005 & 0/20 \\
\midrule
\multirow{6}{*}{Mammography}
& \textbf{ECOD}    & 0.504 $\pm$ 0.053 & 0.912 $\pm$ 0.017 & 0.067 $\pm$ 0.005 & 17/20 \\
& \textbf{COPOD}   & 0.502 $\pm$ 0.053 & 0.911 $\pm$ 0.018 & 0.077 $\pm$ 0.006 & 3/20 \\
& LODA    & 0.349 $\pm$ 0.040 & 0.887 $\pm$ 0.017 & 0.043 $\pm$ 0.007 & 0/20 \\
& IForest & 0.294 $\pm$ 0.048 & 0.882 $\pm$ 0.017 & 0.099 $\pm$ 0.006 & 0/20 \\
& INNE    & 0.275 $\pm$ 0.036 & 0.842 $\pm$ 0.019 & 0.289 $\pm$ 0.003 & 0/20 \\
& HBOS    & 0.146 $\pm$ 0.026 & 0.844 $\pm$ 0.018 & 0.106 $\pm$ 0.009 & 0/20 \\
\bottomrule
\end{tabularx}
\end{table*}

\begin{table*}[tp]
	\centering
	\caption{\textbf{Performance of conformal inference strategies across anomaly detection benchmarks.}
    All weighted methods employ WCS with \textit{\ul{deterministic} pruning} to guarantee finite-sample FDR control.
    Values represent the mean $\pm$ standard deviation of the empirical marginal FDR and statistical power aggregated over 20 independent trials with randomized splits.
    We compare deterministic and randomized baselines against the proposed continuous (KDE-based) approach.
    The calibration and test set sizes are denoted by $n_{\text{train}}$ and $n_{\text{test}}$. Underlined values indicate validity violations (see Section~\ref{sec:metrics}).}

	\label{tab:results_deterministic}
	\renewcommand{\arraystretch}{1.15}
	\setlength{\tabcolsep}{4.5pt}
	\begin{tabularx}{\textwidth}{l l c c c c p{0.5em} C C}
		\toprule
		\textbf{Dataset} & \textbf{Method} & \multicolumn{2}{c}{\textbf{Deterministic}} & \multicolumn{2}{c}{\textbf{Randomized}} & & $n_{\text{train}}$ & $n_{\text{test}}$\\
		\cmidrule(lr){3-4}\cmidrule(lr){5-6}
		                                          & \textit{Deterministic} & \textbf{FDR}    & \textbf{Power}  & \textbf{FDR}    & \textbf{Power}  &   &                        &                        \\
		\midrule
		\multirow{4}{*}{WBC}                      & EDF                    & $0.000\pm0.000$ & $0.000\pm0.000$ & $0.062\pm0.129$ & $0.200\pm0.332$ &   & \multirow{4}{*}{106}   & \multirow{4}{*}{56}    \\
		                                          & Weighted EDF           & $0.000\pm0.000$ & $0.000\pm0.000$ & $0.050\pm0.224$ & $0.017\pm0.074$ &   &                        &                        \\
		& \textbf{KDE} & $0.095\pm0.152$ & $0.500\pm0.351$ & \multicolumn{2}{c}{\textemdash} &  &  &  \\
		& \textbf{Weighted KDE} & $0.078\pm0.142$ & $0.417\pm0.373$ & \multicolumn{2}{c}{\textemdash} &  &  &  \\
		\cmidrule(lr){1-9}
		\multirow{4}{*}{Ionosphere}               & EDF                    & $0.000\pm0.000$ & $0.000\pm0.000$ & $0.076\pm0.199$ & $0.150\pm0.221$ &   & \multirow{4}{*}{112}   & \multirow{4}{*}{88}    \\
		                                          & Weighted EDF           & $0.000\pm0.000$ & $0.000\pm0.000$ & $0.000\pm0.000$ & $0.025\pm0.077$ &   &                        &                        \\
		& \textbf{KDE} & $0.081\pm0.150$ & $0.300\pm0.434$ & \multicolumn{2}{c}{\textemdash} &  &  &  \\
		& \textbf{Weighted KDE} & $0.047\pm0.119$ & $0.138\pm0.339$ & \multicolumn{2}{c}{\textemdash} &  &  &  \\
		\cmidrule(lr){1-9}
		\multirow{4}{*}{WDBC}                     & EDF                    & $0.098\pm0.165$ & $0.280\pm0.442$ & \ul{$0.166\pm0.197$} & $0.440\pm0.452$ &   & \multirow{4}{*}{178}   & \multirow{4}{*}{92}    \\
		                                          & Weighted EDF           & $0.000\pm0.000$ & $0.000\pm0.000$ & $0.000\pm0.000$ & $0.010\pm0.045$ &   &                        &                        \\
		& \textbf{KDE} & $0.086\pm0.135$ & $0.390\pm0.402$ & \multicolumn{2}{c}{\textemdash} &  &  &  \\
		& \textbf{Weighted KDE} & $0.095\pm0.164$ & $0.350\pm0.383$ & \multicolumn{2}{c}{\textemdash} &  &  &  \\
		\cmidrule(lr){1-9}
		\multirow{4}{*}{\shortstack{Breast Cancer\\(Wisconsin)}}
& EDF & $0.000\pm0.000$ & $0.000\pm0.000$ & $0.000\pm0.000$ & $0.094\pm0.145$ &  & \multirow{4}{*}{222} & \multirow{4}{*}{171} \\
		                                          & Weighted EDF           & $0.000\pm0.000$ & $0.000\pm0.000$ & $0.000\pm0.000$ & $0.006\pm0.025$ &   &                        &                        \\
		& \textbf{KDE} & $0.046\pm0.074$ & $0.350\pm0.296$ & \multicolumn{2}{c}{\textemdash} &  &  &  \\
		& \textbf{Weighted KDE} & $0.027\pm0.066$ & $0.267\pm0.218$ & \multicolumn{2}{c}{\textemdash} &  &  &  \\
		\cmidrule(lr){1-9}
		\multirow{4}{*}{Vowels}                   & EDF                    & $0.000\pm0.000$ & $0.000\pm0.000$ & $0.017\pm0.074$ & $0.067\pm0.071$ &   & \multirow{4}{*}{703}   & \multirow{4}{*}{364}   \\
		                                          & Weighted EDF           & $0.000\pm0.000$ & $0.000\pm0.000$ & $0.000\pm0.000$ & $0.000\pm0.000$ &   &                        &                        \\
		& \textbf{KDE} & $0.035\pm0.109$ & $0.122\pm0.082$ & \multicolumn{2}{c}{\textemdash} &  &  &  \\
		& \textbf{Weighted KDE} & $0.035\pm0.109$ & $0.117\pm0.082$ & \multicolumn{2}{c}{\textemdash} &  &  &  \\
		\cmidrule(lr){1-9}
		\multirow{4}{*}{Cardio}                   & EDF                    & $0.034\pm0.091$ & $0.039\pm0.097$ & \ul{$0.119\pm0.250$} & $0.067\pm0.096$ &   & \multirow{4}{*}{827}   & \multirow{4}{*}{458}   \\
		                                          & Weighted EDF           & $0.018\pm0.081$ & $0.015\pm0.068$ & $0.018\pm0.081$ & $0.024\pm0.068$ &   &                        &                        \\
		& \textbf{KDE} & $0.066\pm0.149$ & $0.089\pm0.079$ & \multicolumn{2}{c}{\textemdash} &  &  &  \\
		& \textbf{Weighted KDE} & $0.035\pm0.098$ & $0.087\pm0.080$ & \multicolumn{2}{c}{\textemdash} &  &  &  \\
		\cmidrule(lr){1-9}
		\multirow{4}{*}{Musk}                     & EDF                    & $0.102\pm0.060$ & $1.000\pm0.000$ & \ul{$0.105\pm0.060$} & $1.000\pm0.000$ &   & \multirow{4}{*}{1,482} & \multirow{4}{*}{766}   \\
		                                          & Weighted EDF           & $0.090\pm0.060$ & $0.950\pm0.224$ & $0.100\pm0.057$ & $1.000\pm0.000$ &   &                        &                        \\
		& \textbf{KDE} & $0.084\pm0.060$ & $1.000\pm0.000$ & \multicolumn{2}{c}{\textemdash} &  &  &  \\
		& \textbf{Weighted KDE} & $0.082\pm0.060$ & $1.000\pm0.000$ & \multicolumn{2}{c}{\textemdash} &  &  &  \\
		\cmidrule(lr){1-9}
		\multirow{4}{*}{Satellite}                & EDF                    & \ul{$0.109\pm0.107$} & $0.259\pm0.082$ & \ul{$0.108\pm0.105$} & $0.267\pm0.085$ &   & \multirow{4}{*}{2,199} & \multirow{4}{*}{1,609} \\
		                                          & Weighted EDF           & $0.094\pm0.105$ & $0.214\pm0.124$ & $0.099\pm0.104$ & $0.233\pm0.113$ &   &                        &                        \\
		& \textbf{KDE} & \ul{$0.117\pm0.103$} & $0.291\pm0.087$ & \multicolumn{2}{c}{\textemdash} &  &  &  \\
		& \textbf{Weighted KDE} & \ul{$0.112\pm0.098$} & $0.284\pm0.070$ & \multicolumn{2}{c}{\textemdash} &  &  &  \\
		\cmidrule(lr){1-9}
		\multirow{4}{*}{Mammography}              & EDF                    & $0.019\pm0.037$ & $0.052\pm0.057$ & $0.026\pm0.040$ & $0.069\pm0.053$ &   & \multirow{4}{*}{5,461} & \multirow{4}{*}{2,796} \\
		                                          & Weighted EDF           & $0.017\pm0.037$ & $0.038\pm0.049$ & $0.017\pm0.037$ & $0.039\pm0.048$ &   &                        &                        \\
		& \textbf{KDE} & $0.045\pm0.060$ & $0.085\pm0.050$ & \multicolumn{2}{c}{\textemdash} &  &  &  \\
		& \textbf{Weighted KDE} & $0.037\pm0.051$ & $0.077\pm0.048$ & \multicolumn{2}{c}{\textemdash} &  &  &  \\
		\bottomrule
	\end{tabularx}
\end{table*}

\begin{table*}[tp]
\centering
\caption{\textbf{Performance of conformal inference strategies across anomaly detection benchmarks.}
All weighted methods employ WCS with \textit{\ul{heterogeneous} pruning} to guarantee finite-sample FDR control.
Values represent the mean $\pm$ standard deviation of the empirical marginal FDR and statistical power aggregated over 20 independent trials with randomized splits.
We compare deterministic and randomized baselines against the proposed continuous (KDE-based) approach.
The calibration and test set sizes are denoted by $n_{\text{train}}$ and $n_{\text{test}}$. Underlined values indicate validity violations (see Section~\ref{sec:metrics}).}

\label{tab:results_heterogeneous}
\renewcommand{\arraystretch}{1.15}
\setlength{\tabcolsep}{4.5pt}
\begin{tabularx}{\textwidth}{l l c c c c p{0.5em} C C}
\toprule
\textbf{Dataset} & \textbf{Method}
& \multicolumn{2}{c}{\textbf{Deterministic}}
& \multicolumn{2}{c}{\textbf{Randomized}}
&
& $n_{\text{train}}$
& $n_{\text{test}}$\\
\cmidrule(lr){3-4}\cmidrule(lr){5-6}
& \textit{Heterogeneous} & \textbf{FDR} & \textbf{Power} & \textbf{FDR} & \textbf{Power} & & \\
\midrule

\multirow{4}{*}{WBC}
& EDF                   & $0.000\pm0.000$ & $0.000\pm0.000$ & $0.062\pm0.129$ & $0.200\pm0.332$ & & \multirow{4}{*}{106} & \multirow{4}{*}{56} \\
& Weighted EDF          & $0.000\pm0.000$ & $0.000\pm0.000$ & $0.075\pm0.245$ & $0.083\pm0.183$ & &  &   \\
& \textbf{KDE}          & $0.095\pm0.152$ & $0.500\pm0.351$ & \multicolumn{2}{c}{\textemdash} & &  &   \\
& \textbf{Weighted KDE} & $0.078\pm0.142$ & $0.417\pm0.373$ & \multicolumn{2}{c}{\textemdash} & &  &   \\
\cmidrule(lr){1-9}

\multirow{4}{*}{Ionosphere}
& EDF                   & $0.000\pm0.000$ & $0.000\pm0.000$ & $0.076\pm0.199$ & $0.150\pm0.221$ & & \multirow{4}{*}{112} & \multirow{4}{*}{88} \\
& Weighted EDF          & $0.000\pm0.000$ & $0.000\pm0.000$ & $0.042\pm0.131$ & $0.075\pm0.143$ & &  &   \\
& \textbf{KDE}          & $0.081\pm0.150$ & $0.300\pm0.434$ & \multicolumn{2}{c}{\textemdash} & & &   \\
& \textbf{Weighted KDE} & $0.047\pm0.119$ & $0.138\pm0.339$ & \multicolumn{2}{c}{\textemdash} & &  &   \\
\cmidrule(lr){1-9}

\multirow{4}{*}{WDBC}
& EDF                   & $0.098\pm0.165$ & $0.280\pm0.442$ & \ul{$0.166\pm0.197$} & $0.440\pm0.452$ & & \multirow{4}{*}{178} & \multirow{4}{*}{92} \\
& Weighted EDF          & $0.000\pm0.000$ & $0.000\pm0.000$ & \ul{$0.142\pm0.231$} & $0.100\pm0.138$ &  &  &  \\
& \textbf{KDE}          & $0.086\pm0.135$ & $0.390\pm0.402$ & \multicolumn{2}{c}{\textemdash} &  &  &  \\
& \textbf{Weighted KDE} & $0.095\pm0.164$ & $0.350\pm0.383$ & \multicolumn{2}{c}{\textemdash} &  &  &  \\
\cmidrule(lr){1-9}

\multirow{4}{*}{\shortstack{Breast Cancer\\(Wisconsin)}}
& EDF                   & $0.000\pm0.000$ & $0.000\pm0.000$ & $0.000\pm0.000$ & $0.094\pm0.145$ & & \multirow{4}{*}{222} & \multirow{4}{*}{171} \\
& Weighted EDF          & $0.000\pm0.000$ & $0.000\pm0.000$ & $0.000\pm0.000$ & $0.044\pm0.084$ &  &  &  \\
& \textbf{KDE}          & $0.046\pm0.074$ & $0.350\pm0.296$ & \multicolumn{2}{c}{\textemdash} &  &  &  \\
& \textbf{Weighted KDE} & $0.027\pm0.066$ & $0.267\pm0.218$ & \multicolumn{2}{c}{\textemdash} &  &  &  \\
\cmidrule(lr){1-9}

\multirow{4}{*}{Vowels}
& EDF                   & $0.000\pm0.000$ & $0.000\pm0.000$ & $0.017\pm0.074$ & $0.067\pm0.071$ & & \multirow{4}{*}{703} & \multirow{4}{*}{364} \\
& Weighted EDF          & $0.000\pm0.000$ & $0.000\pm0.000$ & $0.000\pm0.000$ & $0.011\pm0.029$ &  &  &  \\
& \textbf{KDE}          & $0.035\pm0.109$ & $0.122\pm0.082$ & \multicolumn{2}{c}{\textemdash} &  &  &  \\
& \textbf{Weighted KDE} & $0.035\pm0.109$ & $0.117\pm0.082$ & \multicolumn{2}{c}{\textemdash} &  &  &  \\
\cmidrule(lr){1-9}

\multirow{4}{*}{Cardio}
& EDF                   & $0.034\pm0.091$ & $0.039\pm0.097$ & $0.119\pm0.250$ & $0.067\pm0.096$ & & \multirow{4}{*}{827} & \multirow{4}{*}{458} \\
& Weighted EDF          & $0.018\pm0.081$ & $0.015\pm0.068$ & $0.035\pm0.107$ & $0.028\pm0.069$ &  & &  \\
& \textbf{KDE}          & $0.066\pm0.149$ & $0.089\pm0.079$ & \multicolumn{2}{c}{\textemdash} &  &  & \\
& \textbf{Weighted KDE} & $0.035\pm0.098$ & $0.087\pm0.080$ & \multicolumn{2}{c}{\textemdash} &  &  &  \\
\cmidrule(lr){1-9}

\multirow{4}{*}{Musk}
& EDF                   & $0.102\pm0.060$ & $1.000\pm0.000$  & \ul{$0.105\pm0.060$} & $1.000\pm0.000$ & & \multirow{4}{*}{1,482} & \multirow{4}{*}{766} \\
& Weighted EDF          & $0.096\pm0.056$ & $1.000\pm0.000$   & $0.103\pm0.056$ & $1.000\pm0.000$ &  &  &  \\
& \textbf{KDE}          & $0.084\pm0.060$ & $1.000\pm0.000$  & \multicolumn{2}{c}{\textemdash} &  &  &  \\
& \textbf{Weighted KDE} & $0.082\pm0.060$ & $1.000\pm0.000$  & \multicolumn{2}{c}{\textemdash} &  &  &  \\
\cmidrule(lr){1-9}

\multirow{4}{*}{Satellite}
& EDF                   & \ul{$0.109\pm0.107$} & $0.259\pm0.082$ & \ul{$0.108\pm0.105$} & $0.267\pm0.085$ & & \multirow{4}{*}{2,199} & \multirow{4}{*}{1,609} \\
& Weighted EDF          & $0.104\pm0.099$ & $0.249\pm0.085$ & \ul{$0.107\pm0.103$} & $0.261\pm0.081$ &  &  &  \\
& \textbf{KDE}          & \ul{$0.117\pm0.103$} & $0.291\pm0.087$ & \multicolumn{2}{c}{\textemdash} &  &  &  \\
& \textbf{Weighted KDE} & \ul{$0.112\pm0.098$} & $0.284\pm0.070$ & \multicolumn{2}{c}{\textemdash} &  &  &  \\
\cmidrule(lr){1-9}

\multirow{4}{*}{Mammography}
& EDF                   & $0.019\pm0.037$ & $0.052\pm0.057$ & $0.026\pm0.040$ & $0.069\pm0.053$ & & \multirow{4}{*}{5,461} & \multirow{4}{*}{2,796} \\
& Weighted EDF          & $0.017\pm0.037$ & $0.038\pm0.049$ & $0.017\pm0.037$ & $0.045\pm0.045$ &  &  &  \\
& \textbf{KDE}          & $0.045\pm0.060$ & $0.085\pm0.050$ & \multicolumn{2}{c}{\textemdash} &  &  &  \\
& \textbf{Weighted KDE} & $0.037\pm0.051$ & $0.077\pm0.048$ & \multicolumn{2}{c}{\textemdash} &  &  & \\
\bottomrule
\end{tabularx}
\end{table*}

\clearpage

\newpage

\section{Proof Sketch of Theorem~\ref{thm:asymp_valid}}
\label{sec:proof_sketch}

\begin{proof}[Proof sketch of Theorem~\ref{thm:asymp_valid}]
The argument proceeds in three steps.

\paragraph{Step 1 (Uniform CDF consistency).}
Decompose via the triangle inequality:
\[
\sup_t\bigl|\widehat F_0^w(t)-F_0^w(t)\bigr|
\le
\underbrace{\sup_t\bigl|\widehat F_0^w(t)-\widetilde F_0^w(t)\bigr|}_{(\mathrm{I})}
+
\underbrace{\sup_t\bigl|\widetilde F_0^w(t)-F_0^w(t)\bigr|}_{(\mathrm{II})},
\]
where $\widetilde F_0^w$ denotes the oracle weighted kernel CDF estimator using the true weights $w$.

For term~$(\mathrm{II})$, write
\[
\sup_t\bigl|\widetilde F_0^w(t)-F_0^w(t)\bigr|
\le
\underbrace{\sup_t\bigl|\widetilde F_0^w(t)-\mathbb E[\widetilde F_0^w(t)\mid w]\bigr|}_{\text{stochastic}}
+
\underbrace{\sup_t\bigl|\mathbb E[\widetilde F_0^w(t)\mid w]-F_0^w(t)\bigr|}_{\text{bias}}.
\]
Under standard regularity conditions for weighted smoothed empirical distribution estimators
(e.g.\ bounded oracle weights, $\sum_i (w_i/\sum_k w_k)^2 = O_p(N^{-1})$, and a bounded monotone kernel CDF $\Phi_K$),
the stochastic term is $O_p(N^{-1/2})$ uniformly in $t$.
If the weighted null density $f_0^w=(F_0^w)'$ has $r$ derivatives and the kernel is of order $r$,
then the bias term is $O(h_N^r)$, hence vanishes when $h_N\to0$.
Therefore
\[
\sup_t\bigl|\widetilde F_0^w(t)-F_0^w(t)\bigr| \to 0
\quad\text{in probability.}
\]

For term~$(\mathrm{I})$, since $\Phi_K\in[0,1]$,
\[
(\mathrm{I})
\le
\sum_{i=1}^N
\left|
\frac{\hat w_i}{\sum_k\hat w_k}
-
\frac{w_i}{\sum_k w_k}
\right|.
\]
Hence it suffices that the normalized weight vectors are $\ell_1$-consistent; for example,
this holds if
\[
\frac1N\sum_{i=1}^N |\hat w_i-w_i| = o_p(1)
\quad\text{and}\quad
\frac1N\sum_{i=1}^N w_i \xrightarrow{p} c>0.
\]
Thus $(\mathrm{I})\xrightarrow{p}0$ as well.

\paragraph{Step 2 ($p$-value consistency).}
Since $\hat p(z)=1-\widehat F_0^w(s(z))$ and $p^*(z)=1-F_0^w(s(z))$:
\[
\sup_z\bigl|\hat p(z)-p^*(z)\bigr|
\;\le\;
\sup_t\bigl|\widehat F_0^w(t)-F_0^w(t)\bigr|
\;\xrightarrow{p}\;0.
\]

\paragraph{Step 3 (Asymptotic super-uniformity).}
Under $H_0$, $Z\sim Q_0$ is independent of the calibration data. Since $F_0^w$ is continuous, the probability integral transform gives $p^*(Z)\sim\mathrm{Unif}[0,1]$. Conditional on the calibration sample:
\[
\mathbb P_Z\!\bigl(\hat p(Z)\le u\mid\mathrm{cal}\bigr)
\;\le\;
\mathbb P_Z\!\bigl(p^*(Z)\le u+\|\hat p-p^*\|_\infty\mid\mathrm{cal}\bigr)
\;=\;
u+\|\hat p-p^*\|_\infty.
\]
Taking expectations over the calibration data and applying dominated convergence (since $\|\hat p-p^*\|_\infty\le 1$):
\[
\mathbb P\bigl(\hat p(Z)\le u\bigr)
\;\le\;
u+\mathbb E\bigl[\|\hat p-p^*\|_\infty\bigr]
\;=\;
u+o(1).
\qedhere
\]
\end{proof}

\clearpage

\newpage

\section{Stabilizing Weight Estimation}
\label{sec:stable_weight_est}

When $|D_{\text{calib}}|$ and $|D_{\text{test}}|$ differ strongly, classifier-based density-ratio (importance-weight) estimation can become high-variance, yielding \textit{spiky} weight distributions. As discussed in Section~\ref{sec:dilemma}, excessively large test weights inflate the lower bound of the conservative weighted conformal $p$-value and can lead to strict power loss. To mitigate this instability, we use a balanced bootstrap bagging scheme followed by mild clipping.

\paragraph{Balanced bootstrap bagging.}
Let $B$ be the number of bootstrap iterations and define the balanced sample size
\[
S=\min\{|D_{\text{calib}}|,\ |D_{\text{test}}|\}.
\]
For each bootstrap iteration $b\in\{1,\dots,B\}$, we sample with replacement $S$ points from $D_{\text{calib}}$ and $S$ points from $D_{\text{test}}$, forming a balanced training set $\mathcal{D}^{(b)}$. We train a probabilistic classifier $g^{(b)}$ to distinguish test versus calibration membership, and then evaluate it on the full original pool
\[
Z = D_{\text{calib}} \cup D_{\text{test}}.
\]
We intentionally evaluate on all of $Z$ (rather than using out-of-bag predictions), prioritizing estimator stability over the potential bias reduction of out-of-bag aggregation.

Let $g^{(b)}(z)$ denote the predicted probability of the test label, i.e.\ $g^{(b)}(z)\approx \mathbb P(Y=1\mid Z=z)$ under the bootstrap training mixture in iteration $b$ (with $Y=1$ indicating test and $Y=0$ indicating calibration). Under standard density-ratio modeling, Bayes' rule yields
\[
\frac{p_{\text{test}}(z)}{p_{\text{calib}}(z)}
=
\frac{\mathbb P(Y=0)}{\mathbb P(Y=1)}\cdot \frac{\mathbb P(Y=1\mid Z=z)}{\mathbb P(Y=0\mid Z=z)}
\approx
\frac{\mathbb P(Y=0)}{\mathbb P(Y=1)}\cdot \frac{g^{(b)}(z)}{1-g^{(b)}(z)}.
\]
Because $\mathcal{D}^{(b)}$ is class-balanced, $\mathbb P(Y=1)=\mathbb P(Y=0)=1/2$ in training, and the prior-ratio factor equals $1$. Hence, the per-bootstrap weight estimate simplifies to
\begin{equation}
\hat w^{(b)}(z)=\frac{g^{(b)}(z)}{1-g^{(b)}(z)}.
\end{equation}
This is a special case of the general density ratio estimate in Section~\ref{subsec:weight_estimation}: since $\mathcal{D}^{(b)}$ is class-balanced, the prior correction factor $N_{\text{cal}}/N_{\text{test}}$ reduces to unity. If a bootstrap replicate is not exactly balanced, we include the corresponding correction factor $\mathbb P(Y=0)/\mathbb P(Y=1)$.

\paragraph{Aggregation across bootstrap replicas.}
We aggregate the bootstrap estimates via geometric averaging (equivalently, averaging in log-space):
\begin{equation}
\hat w_{\mathrm{bag}}(z)
=
\exp\!\left(\frac{1}{B}\sum_{b=1}^{B}\log \hat w^{(b)}(z)\right).
\end{equation}
This aggregation is natural for positive, multiplicative quantities such as density ratios and reduces the effect of rare extreme values produced by individual bootstrap classifiers.

\paragraph{Winsorization for numerical stability.}
Finally, we winsorize $\hat w_{\mathrm{bag}}$ by clipping to empirical quantiles. Let $q_\gamma$ and $q_{1-\gamma}$ be the empirical $\gamma$ and $(1-\gamma)$ quantiles of $\{\hat w_{\mathrm{bag}}(z): z\in Z\}$. Define
\begin{equation}
\hat w_{\mathrm{final}}(z)
=
\min\!\left\{q_{1-\gamma},\max\!\left\{q_{\gamma},\hat w_{\mathrm{bag}}(z)\right\}\right\}.
\end{equation}
This step should be viewed as a variance-reduction regularization layer: it limits extreme weights in highly skewed settings at the possible cost of a small clipping-induced bias.

\paragraph{Fairness across methods.}
Crucially, for all weighted methods we reuse the \emph{exact same} stabilized weights $\hat w_{\mathrm{final}}(z)$ (from the same bagging and winsorization procedure) across evaluations, ensuring strictly comparable results that are not confounded by method-specific weight estimation noise.

\end{document}